\documentclass{article}
\usepackage{times}
\usepackage{balance}  
\usepackage{tikz}
\usetikzlibrary{positioning,patterns,shapes}
\usepackage{pgfplots}
\usepackage[ruled,linesnumbered]{algorithm2e}
\usepackage[colorlinks,linkcolor=red,citecolor=blue,urlcolor=blue]{hyperref}
\usepackage{url}
\usepackage{multirow}
\usepackage{hhline}
\usepackage{graphicx}
\usepackage{caption}
\usepackage{subcaption}
\usepackage{nips14submit_e,times} 

\input{Definitions}

\title{DFacTo: Distributed Factorization of Tensors}

\author{
Joon Hee Choi \\
Electrical \& Computer Engineering \\
Purdue University \\
West Lafayette IN 47907 \\
\texttt{choi240@purdue.edu} \\
\And
S.~V.~N.~Vishwanathan \\
Statistics and Computer Science \\
Purdue University \\
West Lafayette IN 47907 \\
\texttt{vishy@stat.purdue.edu} \\
}

\nipsfinalcopy
\begin{document}

\maketitle

\begin{abstract}
  We present a technique for significantly speeding up Alternating Least
  Squares (ALS) and Gradient Descent (GD), two widely used algorithms
  for tensor factorization. By exploiting properties of the Khatri-Rao
  product, we show how to efficiently address a computationally
  challenging sub-step of both algorithms. Our algorithm, DFacTo, only
  requires two sparse matrix-vector products and is easy to
  parallelize. DFacTo is not only scalable but also on average 4 to 10
  times faster than competing algorithms on a variety of datasets. For
  instance, DFacTo only takes 480 seconds on 4 machines to perform one
  iteration of the ALS algorithm and 1,143 seconds to perform one
  iteration of the GD algorithm on a 6.5 million $\times$ 2.5 million
  $\times$ 1.5 million dimensional tensor with \textbf{1.2 billion}
  non-zero entries.
\end{abstract}

\section{Introduction}
\label{sec:Introduction}

Tensor data appears naturally in a number of applications
\cite{SmiBroGel04,KolBad09}. For instance, consider a social network
evolving over time. One can form a users $\times$ users $\times$ time
tensor which contains snapshots of interactions between members of the
social network ~\cite{LesKleFal05}. As another example consider an online
store such as Amazon.com where users routinely review various
products. One can form a users $\times$ items $\times$ words tensor from
the review text \cite{McALes13}. Similarly a tensor can be formed by
considering the various contexts in which a user has interacted with an
item \cite{KarAmaBalOli10}. Finally, consider data collected by the
Never Ending Language Learner from the Read the Web project which
contains triples of noun phrases and the context in which they occur,
such as, (``George Harrison'', ``plays'', ``guitars'')
\cite{CarBetKisSetetal10}. 

While matrix factorization and matrix completion have become standard
tools that are routinely used by practitioners, unfortunately, the same
cannot be said about tensor factorization. The reasons are not very hard
to see: There are two popular algorithms for tensor factorization namely
Alternating Least Squares (ALS) (Appendix~\ref{sec:ALS}), and Gradient
Descent (GD) (Appendix~\ref{sec:GD}). The key step in both algorithms is
to multiply a matricized tensor and a Khatri-Rao product of two matrices
(line~\ref{algo:x1cb-als} of Algorithm~\ref{algo:parfac} and line~\ref{algo:x1cb-gd} of Algorithm~\ref{algo:gd}).

However, this process leads to a computationally-challenging,
intermediate data explosion problem. This problem is exacerbated when the
dimensions of tensor we need to factorize are very large (of the order
of hundreds of thousands or millions), or when sparse tensors contain
millions to billions of non-zero entries. For instance, a tensor we
formed using review text from Amazon.com has dimensions of 6.5 million
$\times$ 2.5 million $\times$ 1.5 million and contains approximately
\textbf{1.2 billion} non-zero entries.

Some studies have identified this intermediate data explosion problem
and have suggested ways of addressing it. First, the Tensor Toolbox
\cite{BadKol07} uses the method of reducing indices of the tensor for
sparse datasets and entrywise multiplication of vectors and matrices for
dense datasets. However, it is not clear how to store data or how to
distribute the tensor factorization computation to multiple machines
(see Appendix \ref{sec:IllustrDiffAlgor}). That is, there is a lack of
\emph{distributable} algorithms in existing studies.  Another possible
strategy to solve the data explosion problem is to use
GigaTensor~\cite{KanPapHarFal12}. Unfortunately, while GigaTensor does
address the problem of parallel computation, it is relatively slow. To
summarize, existing algorithms for tensor factorization such as the
excellent Tensor Toolbox of \cite{BadKol07}, or the Map-Reduce based
GigaTensor algorithm of \cite{KanPapHarFal12} often do not scale to
large problems.

In this paper, we introduce an efficient, scalable and distributed
algorithm, DFacTo, that addresses the data explosion problem.  Since
most large-scale real datasets are sparse, we will focus exclusively on
sparse tensors. This is well justified because previous studies have
shown that designing specialized algorithms for sparse tensors can yield
significant speedups \cite{BadKol07}. We show that DFacTo can be
applied to both ALS and GD, and naturally lends itself to a distributed
implementation. Therefore, it can be applied to massive real
datasets which cannot be stored and manipulated on a single machine. For
ALS, DFacTo is on average around 5 times faster than GigaTensor and
around 10 times faster than the Tensor Toolbox on a variety of
datasets. In the case of GD, DFacTo is on average around 4 times faster
than CP-OPT~\cite{AcaDunKol11} from the Tensor Toolbox. On the
Amazon.com review dataset, DFacTo only takes 480 seconds on 4 machines
to perform one iteration of ALS and 1,143 seconds to perform one
iteration of GD.

As with any algorithm, there is a trade-off: DFacTo uses 3 times more memory than the Tensor Toolbox, since it needs to store 3
flattened matrices as opposed to a single tensor.  However, in return,
our algorithm only requires two sparse matrix-vector multiplications,
making DFacTo easy to implement using any standard sparse linear algebra
library. Therefore, there are two merits of using our algorithm: 1)
computations are distributed in a natural way; and 2) only standard
operations are required.

\section{Notation and Preliminaries}
\label{sec:NotatPrel}

Our notation is standard, and closely follows \cite{KolBad09}. Also see
\cite{SmiBroGel04}. Lower case letters such as $x$ denote scalars, bold
lower case letters such as $\xb$ denote vectors, bold upper case letters
such as $\Xb$ represent matrices, and calligraphic letters such as
$\Xscrb$ denote three-dimensional tensors.

The $i$-th element of a vector $\xb$ is written as $x_{i}$. In a similar
vein, the $(i,j)$-th entry of a matrix $\Xb$ is denoted as $x_{i,j}$ and
the $(i, j, k)$-th entry of a tensor $\Xscrb$ is written as
$x_{i,j,k}$. Furthermore, $\xb_{i,:}$ (resp.\ $\xb_{:,i}$) denotes the
$i$-th row (resp.\ column) of $\Xb$. We will use $\Xb_{\Omega, :}$
(resp.\ $\Xb_{:,\Omega}$) to denote the sub-matrix of $\Xb$ which
contains the rows (resp.\ columns) indexed by the set $\Omega$. For
instance, if $\Omega = \cbr{2,4}$, then $\Xb_{\Omega,:}$ is a matrix
which contains the second and fourth rows of $\Xb$. Extending the above
notation to tensors, we will write $\Xb_{i,:,:}$, $\Xb_{:,j,:}$ and
$\Xb_{:,:,k}$ to respectively denote the horizontal, lateral and frontal
\emph{slices} of a third-order tensor $\Xscrb$. The column, row, and tube
\emph{fibers} of $\Xscrb$ are given by $\xb_{:,j,k}$, $\xb_{i,:,k}$, and
$\xb_{i,j,:}$ respectively. 

Sometimes a matrix or tensor may not be fully observed. We will use
$\Omega^{\Xb}$ or $\Omega^{\Xscrb}$ respectively to denote the set of
indices corresponding to the observed (or equivalently non-zero) entries
in a matrix $\Xb$ or a tensor $\Xscrb$.  Extending this notation,
$\Omega^{\Xb}_{i,:}$ (resp.\ $\Omega^{\Xb}_{:,j}$) denotes the set of
column (resp.\ row) indices corresponding to the observed entries in the
$i$-th row (resp.\ $j$-th column) of $\Xb$. We define
$\Omega^{\Xscrb}_{i,:,:}$, $\Omega^{\Xscrb}_{:,j,:}$, and
$\Omega^{\Xscrb}_{:,:,k}$ analogously as the set of indices
corresponding to the observed entries of the $i$-th horizontal, $j$-th
lateral, or $k$-th frontal slices of $\Xscrb$.  Also, $nnzr(\Xb)$
(resp.\ $nnzc(\Xb)$) denotes the number of rows (resp.\ columns) of
$\Xb$ which contain at least one non-zero element.

$\Xb^{\top}$ denotes the transpose, $\Xb^{\dagger}$ denotes the
Moore-Penrose pseudo-inverse, and $\nbr{\Xb}$ (resp.\ $\nbr{\Xscrb}$)
denotes the Frobenius norm of a matrix $\Xb$ (resp.\ tensor $\Xscrb$)
\cite{HorJoh90}. Given a matrix $\Ab \in \RR^{n \times m}$, the linear
operator $vec(\Ab)$ yields a vector $\xb \in \RR^{nm}$, which is
obtained by stacking the columns of $\Ab$. On the other hand, given a
vector $\xb \in \RR^{nm}$, the operator $unvec_{(n,m)}(\xb)$ yields a
matrix $\Ab \in \RR^{n \times m}$.

$\Ab \otimes \Bb$ denotes the Kronecker product, $\Ab \odot \Bb$ the
Khatri-Rao product, and $\Ab \ast \Bb$ the Hadamard product of matrices
$\Ab$ and $\Bb$. The outer product of vectors $\ab$ and $\bb$ is written
as $\ab \circ \bb$ (see \eg, \cite{Bernstein05}). Definitions of these
standard matrix products can be found in
Appendix~\ref{app:MatrixProducts}.

\subsection{Flattening Tensors }
\label{sec:FlatteningTensors}

Just like the $vec(\cdot)$ operator flattens a matrix, a tensor $\Xscrb$
may also be unfolded or flattened into a matrix in three ways namely by
stacking the horizontal, lateral, and frontal slices. We use $\Xb^n$ to
denote the $n$-mode flattening of a third-order tensor $\Xscrb \in
\RR^{I \times J \times K}$; $\Xb^{1}$ is of size $I \times JK$,
$\Xb^{2}$ is of size $J \times KI$, and $\Xb^{3}$ is of size $K \times
IJ$. The following relationships hold between the entries of $\Xscrb$
and its unfolded versions (see Appendix~\ref{app:FlatteningTensors} for
an illustrative example):
\begin{align}
  \label{eq:x1}
  x_{i,j,k}  = x^{1}_{i,j+(k-1)J}   = x^{2}_{j,k+(i-1)K} = x^{3}_{k,i+(j-1)I}. 
\end{align}

We can view $\Xb^{1}$ as consisting of $K$ stacked frontal slices of
$\Xscrb$, each of size $I \times J$. Similarly, $\Xb^{2}$ consists of
$I$ slices of size $J \times K$ and $\Xb^{3}$ is made up of $J$ slices
of size $K \times I$. If we use $\Xb^{n,m}$ to denote the $m$-th slice
in the $n$-mode flattening of $\Xscrb$, then observe that the following
holds:
\begin{align}  
  \label{eq:x1-blockIndices}
  x^{1}_{i,j+(k-1)J}  = x^{1,k}_{i,j}, \quad  
  x^{2}_{j,k+(i-1)K}  = x^{2,i}_{j,k}, \quad
  x^{3}_{k,i+(j-1)I}  = x^{3,j}_{k,i}. 
\end{align}
One can state a relationship between the rows and columns of various
flattenings of a tensor, which will be used to derive our distributed
tensor factorization algorithm in Section~\ref{sec:NewModel}. The proof of
the below lemma is in Appendix~\ref{app:proof-lem1}.
\begin{lemma}
  \label{lem:flat-vec}
  Let $(n, n') \in \cbr{ (2,1), (3,2), (1,3)}$, and let $\Xb^{n}$ and
  $\Xb^{n'}$ be the $n$ and $n'$-mode flattening respectively of a
  tensor $\Xscrb$. Moreover, let $\Xb^{n,m}$ be the $m$-th slice in
  $\Xb^{n}$, and $\xb_{m,:}^{n'}$ be the $m$-th row of $\Xb^{n'}$. Then,
  $vec(\Xb^{n,m}) = \xb_{m,:}^{n'}$.
\end{lemma}

\section{DFacTo}
\label{sec:NewModel}

Recall that the main challenge of implementing ALS or GD for solving
tensor factorization lies in multiplying a matricized tensor and a
Khatri-Rao product of two matrices: $\Xb^{1} \rbr{\Cb \odot
  \Bb}$\footnote{We mainly concentrate on the update to $\Ab$ since the
  updates to $\Bb$ and $\Cb$ are analogous.}  . If $\Bb$ is of size $J
\times R$ and $\Cb$ is of size $K \times R$, explicitly forming
$\rbr{\Cb \odot \Bb}$ requires $O(JKR)$ memory and is infeasible when
$J$ and $K$ are large. This is called the intermediate data explosion
problem in the literature~\cite{KanPapHarFal12}. The lemma below will be
used to derive our efficient algorithm, which avoids this
problem. Although the proof can be inferred using results in
\cite{KolBad09}, we give an elementary proof for completeness.
\begin{lemma}
  \label{lem:newparfac}
  The $r$-th column of $\Xb^{1} \rbr{\Cb \odot \Bb}$ can be computed as
  \begin{align}
    \label{eq:vecabc-appl}
    \sbr{\Xb^{1} \rbr{\Cb \odot \Bb}}_{:,r} = \sbr{unvec_{(K,I)} \rbr{
      \rbr{\Xb^{2}}^{\top} \bb_{:,r}}}^{\top} \cbb_{:,r}
  \end{align}
\end{lemma}
\begin{proof}
  We need to show that 
  \begin{align*}
    \sbr{\Xb^{1} \rbr{\Cb \odot \Bb}}_{:,r}
    & = \sbr{unvec_{(K,I)} \rbr{\rbr{\Xb^2}^{\top} \bb_{:,r}}}^{\top}
    \cbb_{:,r} \\
    & = \mymatrix{c}{\bb_{:,r}^{\top} \; \Xb^{2,1} \; \cbb_{:,r} \\
      \vdots \\
      \bb_{:,r}^{\top} \; \Xb^{2,I} \; \cbb_{:,r}}.
  \end{align*}
  Or equivalently it suffices to show that $\sbr{\Xb^{1} \rbr{\Cb \odot
      \Bb}}_{i,r} = \bb_{:,r}^{\top} \; \Xb^{2,i} \; \cbb_{:,r}$.  Using
  \eqref{eq:vecabc}
  \begin{align}
    vec \rbr{\bb_{:,r}^{\top} \; \Xb^{2,i} \; \cbb_{:,r}}
    = \rbr{\cbb_{:,r}^{\top} \otimes \bb_{:,r}^{\top}} vec
    \rbr{\Xb^{2,i}}.
  \end{align}
  Observe that $\bb_{:,r}^{\top} \; \Xb^{2,i} \; \cbb_{:,r}$ is a
  scalar.  Moreover, using Lemma~\ref{lem:flat-vec} we can write $vec
  \rbr{\Xb^{2,i}} = \xb_{i,:}^{1}$. This allows us to rewrite the above
  equation as
  \begin{align*}
    \bb_{:,r}^{\top} \; \Xb^{2,i} \; \cbb_{:,r}  =
    \rbr{\xb_{i,:}^{1}}^{\top} \rbr{\cbb_{:,r} \otimes \bb_{:,r}} =
    \sbr{\Xb^{1} \rbr{\Cb \odot \Bb}}_{i,r},
  \end{align*}
  which completes the proof.
\end{proof}
Unfortunately, a naive computation of $\sbr{\Xb^{1} \rbr{\Cb \odot
    \Bb}}_{:,r}$ by using \eqref{eq:vecabc-appl} does not solve the
intermediate data explosion problem. This is because
$\rbr{\Xb^{2}}^{\top} \bb_{:,r}$ produces a $K I$ dimensional vector,
which is then reshaped by the $unvec_{\rbr{K,I}}(\cdot)$ operator into a
$K \times I$ matrix. However, as the next lemma asserts, only a small
number of entries of $\rbr{\Xb^{2}}^{\top} \bb_{:,r}$ are non-zero.

For convenience, let a vector produced by $(\Xb^2)^{\top} \bb_{:,r}$ be 
$\vb_{:,r}$ and a matrix produced by $\sbr{unvec_{(K,I)}(\vb_{:,r})}^{\top}$ 
be $\Mb^r$.
\begin{lemma}
  \label{lem:nnzc}
  The number of non-zeros in $\vb_{:,r}$ is at most $nnzr((\Xb^2)^{\top})$ 
  and $nnzc(\Xb^2)$.
\end{lemma}
\begin{proof}
  Multiplying an all-zero row in $(\Xb^2)^{\top}$ and $\bb_{:,r}$
  produces zero. Therefore, the number of non-zeros in $\vb_{:,r}$ is
  equal to the number of rows in $(\Xb^2)^{\top}$ that contain at least
  one non-zero element. Also, by definition, $nnzr((\Xb^2)^{\top})$ is
  equal to $nnzc(\Xb^2)$.
\end{proof}
As a consequence of the above lemma, we only need to explicitly compute
the non-zero entries of $\vb_{:,r}$. However, the problem of reshaping
$\vb_{:,r}$ via the $\sbr{unvec_{\rbr{K, I}}(\cdot)}^{\top}$ operator
still remains. The next lemma shows how to overcome this difficulty.
\begin{lemma}
  \label{lem:sparsity}
  The location of the non-zero entries of $\Mb^r$ depends on
  $(\Xb^2)^{\top}$ and is independent of $\bb_{:,r}$.
\end{lemma}
\begin{proof}
  The product of the $(k+(i-1)K)$-$th$ row of $(\Xb^2)^{\top}$ and
  $\bb_{:,r}$ is the $(k+(i-1)K)$-$th$ element of $\bb_{:,r}$. And, this
  element is the $(i,k)$-$th$ entry of $\Mb^r$ by definition of
  $\sbr{unvec_{\rbr{K,I}}(\cdot)}^{\top}$.  Therefore, if all the
  entries in the $(k+(i-1)K)$-$th$ row of $(\Xb^2)^{\top}$ are zero,
  then the $(i,k)$-$th$ entry of $\Mb^r$ is zero regardless of
  $\bb_{:,r}$. Consequently, the location of the non-zero entries of
  $\Mb^r$ is independent of $\bb_{:,r}$, and is only determined by
  $(\Xb^2)^{\top}$.
\end{proof}
Given $\Xscrb$ one can compute $(\Xb^2)^{\top}$ to know the locations of
the non-zero entries of $\Mb^r$. In other words, we can infer the
non-zero pattern and therefore preallocate memory for $\Mb^r$. We will
show below how this allows us to perform the $\sbr{unvec_{\rbr{K,
      I}}(\cdot)}^{\top}$ operation for free.

Recall the Compressed Sparse Row (CSR) Format, which stores a sparse
matrix as three arrays namely \emph{values}, \emph{columns}, and
\emph{rows}.  Here, \emph{values} represents the non-zero values of the
matrix; while \emph{columns} stores the column indices of the non-zero
values. Also, \emph{rows} stores the indices of the \emph{columns} array
where each row starts. For example, if a sparse matrix $\Mb^r$ is
\begin{align*}
  \Mb^r =  \mymatrix{ccc}{1 & 0 & 2 \\ 0 & 3 & 4},
\end{align*}
then the CSR of $\Mb^r$ is
\begin{align*}
  value(\Mb^r) &= \mymatrix{cccc}{1 & 2 & 3 & 4} \\
  col(\Mb^r) &= \mymatrix{cccc}{0 & 2 & 1 & 2} \\
  row(\Mb^r) &= \mymatrix{ccc}{0 & 2 & 4}.
\end{align*}
Different matrices with the same sparsity pattern can be represented by
simply changing the entries of the \emph{value} array. For our
particular case, what this means is that we can pre-compute $col(\Mb^r)$
and $row(\Mb^r)$ and pre-allocate $value(\Mb^r)$. By writing the
non-zero entries of $\vb_{:,r}$ into $value(\Mb^{r})$ we can ``reshape''
$\vb_{:,r}$ into $\Mb^{r}$. 

Let the matrix with all-zero rows in $(\Xb^2)^{\top}$ removed be
$(\Xbhat^2)^{\top}$. Then, Algorithm~\ref{algo:dfacto} shows the DFacTo
algorithm for computing $\Nb := \Xb^{1} \rbr{\Cb \odot \Bb}$. Here, the
input values are $(\Xbhat^2)^{\top}$, $\Bb$, $\Cb$, and $\Mb^r$
preallocated in CSR format. By storing the results of the product of
$(\Xbhat^2)^{\top}$ and $\bb_{:,r}$ directly into $value(\Mb^r)$, we can
obtain $\Mb^r$ because $\Mb^r$ was preallocated in the CSR format. Then,
the product of $\Mb^r$ and $\cbb_{:,r}$ yields the $r$-$th$ column of
$\Nb$. We obtain the output $\Nb$ by repeating these two sparse
matrix-vector products $R$ times.
\begin{algorithm}
  \SetKw{KwOutput}{Output:}
  \SetKw{KwInput}{Input:}
  
  \KwInput $(\Xbhat^{2})^{\top}$, $\Bb$, $\Cb$, $value(\Mb^r)$
  $col(\Mb^{r})$, $row(\Mb^{r})$ 
  
  \KwOutput $\Nb$

  \While{r=1, 2,\ldots, R}{
      $value(\Mb^r) \leftarrow (\Xbhat^{2})^{\top} \; \bb_{:,r}$ \\
      $\nbb_{:,r} \leftarrow \Mb^r \; \cbb_{:,r}$
  }
  \caption{DFacTo algorithm for Tensor Factorization}
  \label{algo:dfacto}
\end{algorithm}

It is immediately obvious that using the above lemmas to compute $\Nb$
requires no extra memory other than storing $\Mb^{r}$, which contains at
most $nnzc(\Xb^{2}) \leq \abr{\Omega^{\Xscrb}}$ non-zero
entries. Therefore, we completely avoid the intermediate data explosion
problem.  Moreover, the same subroutine can be used for both ALS and GD
(see Appendix~\ref{sec:dfacto-alsgd} for detailed pseudo-code).

\subsection{Distributed Memory Implementation}
\label{sec:DistrMemoryImpl}

Our algorithm is easy to parallelize using a master-slave architecture.
At every iteration, the master transmits $\Ab$, $\Bb$, and $\Cb$ to the
slaves. The slaves hold a fraction of the rows of $\Xb^{2}$ using which
a fraction of the rows of $\Nb$ is computed. By performing a
synchronization step, the slaves can exchange rows of $\Nb$. In ALS,
this $\Nb$ is used to compute $\Ab$ which is transmitted back to the
master. Then, the master updates $\Ab$, and the iteration proceeds. In
GD, the slaves transmit $\Nb$ back to the master, which computes $\nabla
\Ab$. Then, the master computes the step size by a line search
algorithm, updates $\Ab$, and the iteration proceeds.

\subsection{Complexity Analysis}
\label{sec:ComplexityAnalysis}

A naive computation of $\Nb$ requires $\rbr{JK+\abr{\Omega^{\Xscrb}}}R$ 
flops; forming $\Cb \odot \Bb$ requires $JKR$ flops and performing the 
matrix-matrix multiplication $\Xb^{1} \rbr{\Cb \odot \Bb}$ requires 
$\abr {\Omega^{\Xscrb}}R$ flops. Our algorithm requires only 
$\rbr{nnzc(\Xb^{2})+\abr{\Omega^{\Xscrb}}}R$ flops; 
$\abr{\Omega^{\Xscrb}}R$ flops for computing $\vb_{:,r}$ and 
$nnzc(\Xb^{2}) R$ flops for computing $\Mb^r \cbb_{:,r}$. Note that, 
typically, $nnzc(\Xb^{2}) \ll$ both $JK$ and $\abr{\Omega^{\Xscrb}}$ 
(see Table~\ref{tab:dataset}).  In terms of memory, the naive algorithm 
requires $O(JKR)$ extra memory, while our algorithm only requires 
$nnzc(\Xb^{2})$ extra space to store $\Mb^{r}$.

\section{Related Work}
\label{sec:RelatedWork}

Two papers that are most closely related to our work are the GigaTensor
algorithm proposed by \cite{KanPapHarFal12} and the Sparse Tensor
Toolbox of \cite{BadKol07}. As discussed above, both algorithms attack
the problem of computing $\Nb$ efficiently. In order to compute
$\nbb_{:, r}$, GigaTensor computes two intermediate matrices $\Nb_{1} :=
\Xb^{1} * \rbr{\one_{I} \odot \rbr{\cbb_{:,r} \otimes \one_{J}}^{\top}}$
and $\Nb_{2} := bin \rbr{\Xb^{1}} * \rbr{\one_{I} \odot \rbr{\one_{K}
    \otimes \bb_{:,r}}^{\top}}$. Next, $\Nb_{3} := \Nb_{1} * \Nb_{2}$ is
computed, and $\nbb_{:, r}$ is obtained by computing $\Nb_{3} \;
\one_{JK}$. As reported in \cite{KanPapHarFal12}, GigaTensor uses $2
\abr{\Omega^{\Xscrb}}$ extra storage and $5 \abr{\Omega^{\Xscrb}}$ flops
to compute one column of $\Nb$. The Sparse Tensor Toolbox stores a
tensor as a vector of non-zero values and a matrix of corresponding
indices. Entries of $\Bb$ and $\Cb$ are replicated appropriately to
create intermediate vectors. A Hadamard product is computed between the
non-zero entries of the matrix and intermediate vectors, and a selected
set of entries are summed to form columns of $\Nb$. The algorithm uses
$2 \abr{\Omega^{\Xscrb}}$ extra storage and $5 \abr{\Omega^{\Xscrb}}$
flops to compute one column of $\Nb$. See
Appendix~\ref{sec:IllustrDiffAlgor} for a detailed illustrative example
which shows all the intermediate calculations performed by our algorithm
as well as the algorithm of \cite{KanPapHarFal12} and
\cite{BadKol07}. 

Also, \cite{AcaDunKol11} suggests the gradient-based optimization algorithm 
of CANDECOMP/PARAFAC (CP) using the same method as \cite{BadKol07} 
to compute $\Xb^{1} \rbr{\Cb \odot \Bb}$. \cite{AcaDunKol11} refers to this
gradient-based optimization algorithm as CPOPT and to the ALS algorithm 
of CP using the method of \cite{BadKol07} as CPALS. Following 
\cite{AcaDunKol11}, we use these names, CPALS and CPOPT.

\section{Experimental Evaluation}
\label{sec:experiments}

Our experiments are designed to study the scaling behavior of DFacTo on
both publicly available real-world datasets as well as synthetically
generated data. We contrast the performance of DFacTo (ALS) with
GigaTensor~\cite{KanPapHarFal12} as well as with 
CPALS~\cite{BadKol07}, while the performance of DFacTo (GD) is 
compared with CPOPT~\cite{AcaDunKol11}. We also present results to 
show the scaling behavior of DFacTo when data is distributed across 
multiple machines.

\paragraph{Datasets}
\label{sec:Datasets}

See Table~\ref{tab:dataset} for a summary of the real-world datasets we
used in our experiments. The NELL-1 and NELL-2 datasets are from
\cite{KanPapHarFal12} and consists of (noun phrase 1, context, noun
phrase 2) triples from the ``Read the Web'' project
\cite{CarBetKisSetetal10}. NELL-2 is a version of NELL-1, which is
obtained by removing entries whose values are below a threshold.

The Yelp Phoenix dataset is from the Yelp Data Challenge
\footnote{https://www.yelp.com/dataset\_challenge/dataset}, while
Cellartracker, Ratebeer, Beeradvocate and Amazon.com are from the
Stanford Network Analysis Project (SNAP) home page. All these datasets
consist of product or business reviews. We converted them into a users
$\times$ items $\times$ words tensor by first splitting the text into
words, removing stop words, using Porter stemming~\cite{Porter80}, and
then removing user-item pairs which did not have any words associated with
them. In addition, for the Amazon.com dataset we filtered words that
appeard less than 5 times or in fewer than 5 documents. Note that the 
number of dimensions as well as the number of non-zero entries reported in
Table~\ref{tab:dataset} differ from those reported in \cite{McALes13}
because of our pre-processing.

\begin{table*}[h]
  \centering
  \begin{tabular}{lrrrrrrr}
    \hline
    Dataset & $I$ & $J$ & $K$ & $\abr{\Omega^{\Xscrbh}}$ & $nnzc(X^{1})$
    & $nnzc(X^{2})$ & $nnzc(X^{3})$ \\ 
    \hline
    Yelp Phoenix  &    45.97K &    11.54K &     84.52K &     9.85M &
    4.32M &   6.11M &    229.83K \\ 
    Cellartracker &    36.54K &   412.36K &    163.46K &    25.02M &
    19.23M &   5.88M &  1.32M \\ 
    NELL-2        &    12.09K &     9.18K &     28.82K &    76.88M &
    16.56M &  21.48M &    337.37K \\ 
    Beeradvocate  &    33.37K &    66.06K &    204.08K &    78.77M &
    18.98M &  12.05M &  1.57M \\ 
    Ratebeer      &    29.07K &   110.30K &    294.04K &    77.13M &
    22.40M &   7.84M &  2.85M \\ 
    \hline
    NELL-1        & 2.90M & 2.14M & 25.50M &   143.68M &
    113.30M & 119.13M & 17.37M \\ 
    Amazon        & 6.64M & 2.44M &  1.68M & 1.22B &
    525.25M & 389.64M & 29.91M \\ 
    \hline
  \end{tabular}
  \caption{Summary statistics of the datasets used in our experiments. }
  \label{tab:dataset}
\end{table*}

We also generated the following two kinds of synthetic data for our
experiments:
\begin{itemize}
\item the number of non-zero entries in the tensor is held fixed but we
  vary $I$, $J$, and $K$.
\item the dimensions $I$, $J$, and $K$ are held fixed but the number of
  non-zeros entries varies.
\end{itemize}
To simulate power law behavior, both the above datasets were generated 
using the following preferential attachment model \cite{BarAlb99}: the 
probability that a non-zero entry is added at index $(i, j, k)$ is given by $p_{i}
\times p_{j} \times p_{k}$, where $p_{i}$ (resp.\ $p_{j}$ and $p_{k}$)
is proportional to the number of non-zero entries at index $i$ (resp.\
$j$ and $k$).

\paragraph{Implementation and Hardware}
\label{sec:Hardware}

All experiments were conducted on a computing cluster where each node
has two 2.1 GHz 12-core AMD 6172 processors with 48 GB physical 
memory per node. Our algorithms are implemented in C++ using the Eigen
library\footnote{\url{http://eigen.tuxfamily.org}} and compiled with the
Intel Compiler. We downloaded Version 2.5 of the Tensor Toolbox, which
is implemented in 
MATLAB\footnote{\url{http://www.sandia.gov/~tgkolda/TensorToolbox/}}. 
Since open source code for GigaTensor is not freely available, we 
developed our own version in C++ following the description
in~\cite{KanPapHarFal12}.

\paragraph{Scaling on Real-World Datasets}
\label{sec:ScalingRealWorld}

Both CPALS and our implementation of GigaTensor are uni-processor codes.
Therefore, for this experiment we restricted ourselves to datasets 
which can fit on a single machine. When initialized with the same starting 
point, DFacTo and its competing algorithms will converge to the same 
solution. Therefore, we only compare the CPU time per iteration of the 
different algorithms. The results are summarized in 
Table~\ref{tab:realdata_itertime}. On many datasets DFacTo (ALS) is around 
5 times faster than GigaTensor and 10 times faster than CPALS; the 
differences are more pronounced on the larger datasets. Also, DFacTo (GD)
is around 4 times faster than CPOPT.

\begin{table}
  \centering
  \begin{tabular}{crrrrr}
    \hline
    Dataset & DFacTo (ALS) & GigaTensor & CPALS & DFacTo (GD) &
      CPOPT \\
    \hline
    Yelp Phoenix  &   9.52 &  26.82 &  46.52 & 13.57 & 45.9 \\
    Cellartracker &  23.89 &  80.65 & 118.25 & 35.82 & 130.32 \\
    NELL-2        &  32.59 & 186.30 & 376.10 & 80.79 & 386.25 \\
    Beeradvocate  &  43.84 & 224.29 & 364.98 & 94.85 & 481.06 \\
    Ratebeer      &  44.20 & 240.80 & 396.63 & 87.36 & 349.18 \\
    NELL-1        & 322.45 & 772.24 &   -    & 742.67 & - \\
    \hline
  \end{tabular}
  \caption{Times per iteration (in seconds) of DFacTo (ALS), GigaTensor, 
    CPALS, DFacTo (GD), and CPOPT on datasets which can fit in a single 
    machine.}
  \label{tab:realdata_itertime}
\end{table}

The difference in performance between DFacTo (ALS) and CPALS and 
between DFacTo (GD) and CPOPT can partially be explained by the fact that 
DFacTo (ALS, GD) is implemented in C++ while CPALS and CPOPT use 
MATLAB. However, it must be borne in mind that both MATLAB and our 
implementation use an optimized BLAS library to perform their 
computationally intensive numerical linear algebra operations.

Compared to the Map-Reduce version implemented in Java and used for the
experiments reported in~\cite{KanPapHarFal12}, our C++ implementation
of GigaTensor is significantly faster and more optimized. As
per~\cite{KanPapHarFal12}, the Java implementation took approximately
10,000 seconds per iteration to handle a tensor with around $10^{9}$
non-zero entries, when using 35 machines. In contrast, the C++ version
was able to handle one iteration of the ALS algorithm on the NELL-1
dataset on a single machine in 772 seconds. However, because 
DFacTo (ALS) uses a better algorithm, it is able to handsomely outperform 
GigaTensor and only takes 322 seconds per iteration.

Also, the execution time of DFacTo (GD) is longer than that of 
DFacTo (ALS) because DFacTo (GD) spends more time on the line search
algorithm to obtain an appropriate step size.

\paragraph{Scaling across Machines}
\label{sec:ScalacrossMach}

Our goal is to study scaling behavior of the time per iteration as
datasets are distributed across different machines. Towards this end we
worked with two datasets. NELL-1 is a moderate-size dataset which our
algorithm can handle on a single machine, while Amazon is a large
dataset which does not fit on a single machine. Table~\ref{tab:parallel}
shows that the iteration time decreases as the number of machines increases on the NELL-1 and Amazon datasets. While the decrease in iteration time is not completely linear, the computation time excluding both synchronization and line search time decreases linearly. The Y-axis in Figure~\ref{fig:su} indicates $T_4 / T_n$ where $T_n$ is the single iteration time with $n$ machines on the Amazon dataset.

\begin{table}
  \centering
  \begin{tabular}{crrrrrrrr}
    & \multicolumn{4}{c}{DFacTo (ALS)} & \multicolumn{4}{c}{DFacTo (GD)} \\
    \hline
    & \multicolumn{2}{c}{NELL-1} & \multicolumn{2}{c}{Amazon} 
    & \multicolumn{2}{c}{NELL-1} & \multicolumn{2}{c}{Amazon} \\
    \hline
    Machines & Iter.\ & CPU & Iter.\ & CPU & Iter.\ & CPU & Iter.\ & CPU \\
    \hline
    1  & 322.45 & 322.45 &   -    &   -   
        & 742.67 & 104.23 & - & - \\
    2  & 205.07 & 167.29 &   -    &   -   
        & 492.38 & 55.11 & - & - \\
    4  & 141.02 & 101.58 & 480.21 & 376.71 
        & 322.65 & 28.55 & 1143.7 & 127.57 \\
    8  &  86.09 &  62.19 & 292.34 & 204.41 
        & 232.41 & 16.24 & 727.79 & 62.61 \\
    16 &  81.24 &  46.25 & 179.23 &  98.07 
         & 178.92 & 9.70 & 560.47 & 28.61 \\
    32 &  90.31 &  34.54 & 142.69 &  54.60 
         & 209.39 & 7.45 & 471.91 & 15.78 \\
    \hline
  \end{tabular}
  \caption{Total Time and CPU time per iteration (in seconds) as a function of 
    number of machines for the NELL-1 and Amazon datasets. }
  \label{tab:parallel}
\end{table}

\begin{figure}[h]
  \centering
  \begin{subfigure}[b]{0.49\textwidth}
    \includegraphics[width=\textwidth]{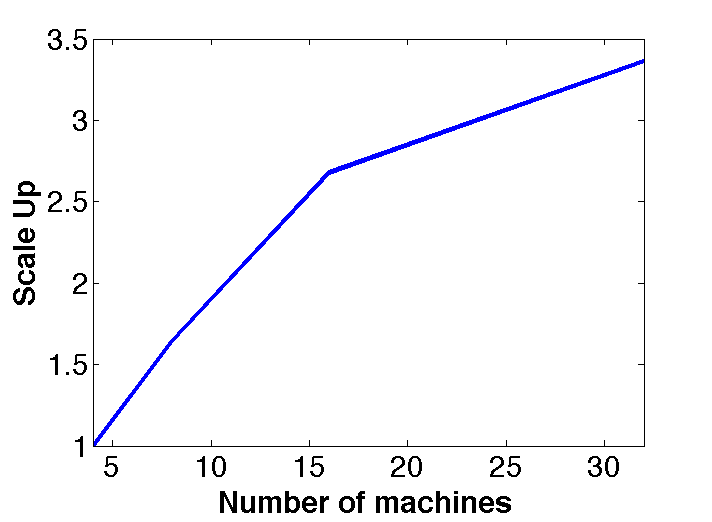}
    \caption{DFacTo(ALS)}
    \label{fig:su-als}
  \end{subfigure}
  \begin{subfigure}[b]{0.49\textwidth}
    \includegraphics[width=\textwidth]{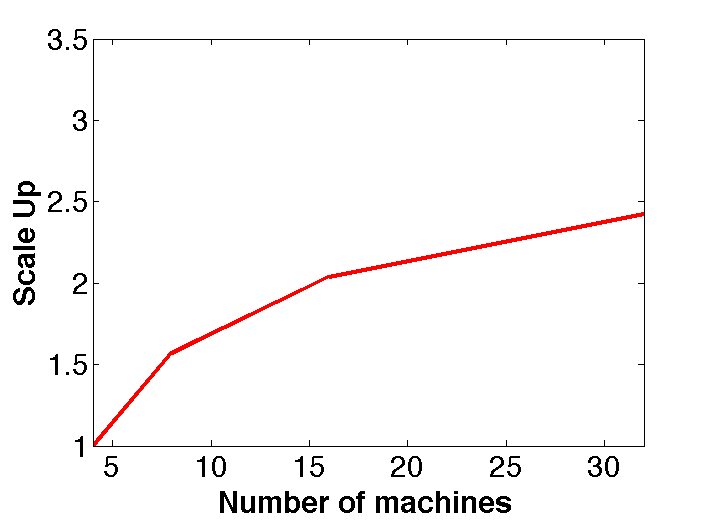}
    \caption{DFacTo(GD)}
    \label{fig:su-gd}
  \end{subfigure}
  \caption{The scalability of DFacTo with respect to the number of machines 
    on the Amazon dataset}
  \label{fig:su}
\end{figure}

\paragraph{Synthetic Data Experiments}
\label{sec:SynthDataExper}

We perform two experiments with synthetically generated tensor data. In
the first experiment we fix the number of non-zero entries to be
$10^{6}$ and let $I=J=K$ and vary the dimensions of the tensor. For the
second experiment we fix the dimensions and let $I=J=K$ and the number
of non-zero entries is set to be $2I$. The scaling behavior of the three
algorithms on these two datasets is summarized in
Tables~\ref{tab:random1_data_itertime}
and~\ref{tab:random2_data_itertime}.  Since we used a preferential
attachment model to generate the datasets, the non-zero indices exhibit
a power law behavior. Consequently, the number of columns with non-zero
elements ($nnzc(\cdot)$) for $\Xb^{1}$, $\Xb^{2}$ and $\Xb^{3}$ is very 
close to the total number of non-zero entries in the tensor. Therefore, as
predicted by theory, DFacTo (ALS, GD) does not enjoy significant speedups 
when compared to GigaTensor, CPALS and CPOPT. However, it must be 
noted that DFacto (ALS) is faster than either GigaTensor or CPALS in all but 
one case and DFacTo (GD) is faster than CPOPT in all cases. We attribute 
this to better memory locality which arises as a consequence of reusing the 
memory for $\Nb$ as discussed in Section~\ref{sec:NewModel}. 

\begin{table}[h]
  \centering
  \begin{tabular}{crrrrr}
    \hline
    $I=J=K$ & DFacTo (ALS) & GigaTensor & CPALS 
      & DFacTo (GD) & CPOPT \\
    \hline
    $10^{4}$ &  1.14 &  2.80 &   5.10 & 2.32 & 5.21 \\
    $10^{5}$ &  2.72 &  6.71 &   6.11 & 5.87 & 11.70 \\
    $10^{6}$ &  7.26 & 11.86 &  16.54 & 16.51 & 29.13 \\
    $10^{7}$ & 41.64 & 38.19 & 175.57 & 121.30& 202.71 \\
    \hline
  \end{tabular}
  \caption{Time per iteration (in seconds) for DFacTo (ALS), GigaTensor, 
    CPALS, DFacTo (GD), and CPOPT on synthetic datasets. The number of 
    non-zero elements is $10^{6}$, and the dimensions of the tensor are varied. }
  \label{tab:random1_data_itertime}
\end{table}

\begin{table}[h]
  \centering
  \begin{tabular}{crrrrr}
    \hline
    $I=J=K$ & DFacTo (ALS) & GigaTensor & CPALS 
      & DFacTo (GD) & CPOPT \\
    \hline
    $10^{4}$ &   0.05 &   0.09 &   0.52 & 0.09 & 0.57 \\
    $10^{5}$ &   0.92 &   1.61 &   1.50 & 1.81 & 2.98 \\
    $10^{6}$ &  12.06 &  22.08 &  15.84 & 21.74 & 26.04 \\
    $10^{7}$ & 144.48 & 251.89 & 214.37 & 275.19 & 324.2 \\
    \hline
  \end{tabular}
  \caption{Time per iteration (in seconds) for DFacTo (ALS), GigaTensor, 
    CPALS, DFacTo (GD), and CPOPT on synthetic datasets. The dimensions 
    of the tensor are varied and the number of non-zero elements is set to $2I$. }
  \label{tab:random2_data_itertime}
\end{table}

\section{Discussion and Conclusion}
\label{sec:Discussion}

We presented a technique for significantly speeding up the Alternating
Least Squares (ALS) and the Gradient Descent (GD) algorithm for tensor
factorization by exploiting properties of the Khatri-Rao product. Not
only is our algorithm, DFacTo, computationally attractive, but it is also
more memory efficient compared to existing algorithms. Furthermore, we
presented a strategy for distributing the computations across multiple
machines.

We hope that the availability of a scalable tensor factorization
algorithm will enable practitioners to work on more challenging tensor
datasets, and therefore lead to advances in the analysis and
understanding of tensor data. Towards this end we intend to make our
code freely available for download under a permissive open source
license.

Although we mainly focused on tensor factorization using ALS and GD, it
is worth noting that one can extend the basic ideas behind DFacTo to
other related problems such as joint matrix completion and tensor
factorization. We present such a model in Appendix
\ref{sec:JointMatcomTenfac}. In fact, we believe that this joint matrix
completion and tensor factorization model by itself is somewhat new and
interesting in its own right, despite its resemblance to other joint
models including tensor factorization such as \cite{AcaKolDun11}. In our
joint model, we are given a user $\times$ item ratings matrix $\Yb$, and
some side information such as a user $\times$ item $\times$ words tensor
$\Xscrb$. Preliminary experimental results suggest that jointly
factorizing $\Yb$ and $\Xscrb$ outperforms vanilla matrix
completion. Please see Appendix~\ref{sec:JointMatcomTenfac} for details
of the algorithm and some experimental results.

\clearpage
\appendix

\section{Matrix Products and Related Identities}
\label{app:MatrixProducts}
\begin{definition}
  The Kronecker product $\Ab \otimes \Bb \in \RR^{mp \times nq}$ of
  matrices $\Ab \in \RR^{m \times n}$ and $\Bb \in \RR^{p \times q}$ is
  defined as
  \begin{align}
    \label{eq:krondef}
    \Ab \otimes \Bb = \mymatrix{cccc}{
      a_{1,1} \Bb & a_{1,2} \Bb & \ldots & a_{1,n} \Bb \\
      \vdots & \vdots & \vdots & \vdots \\
      a_{m,1} \Bb & a_{m,2} \Bb & \ldots & a_{m,n} \Bb}. 
  \end{align}  
\end{definition}
\begin{definition}
  The Khatri-Rao product $\Ab \odot \Bb \in \RR^{mp \times n}$ of
  matrices $\Ab \in \RR^{m \times n}$ and $\Bb \in \RR^{p \times n}$ is
  given by the Kronecker product of the corresponding columns of the two
  matrices:
  \begin{align}
    \label{eq:khatraodef}
    \Ab \odot \Bb = \mymatrix{cccc}{ \ab_{:,1} \otimes \bb_{:,1} &
      \ab_{:,2} \otimes \bb_{:,2} & \ldots & \ab_{:,n} \otimes \bb_{:,n}
    }.
  \end{align}    
\end{definition}
\begin{definition}
  The Hadamard product $\Ab \ast \Bb \in \RR^{n \times m}$  of two
  conforming matrices $\Ab \in \RR^{n \times m}$ and $\Bb \in \RR^{n
    \times m}$ is given by 
  \begin{align}
    \label{eq:hadamard}
    \Ab \ast \Bb = \mymatrix{cccc}{a_{1,1} b_{1,1} & a_{1,2} b_{1,2} &
      \ldots & a_{1,m} b_{1,m} \\
      \vdots & \vdots & \vdots & \vdots \\
      a_{n,1} b_{n,1} & a_{n,2} b_{n,2} &
      \ldots & a_{n,m} b_{n,m}}
  \end{align}
\end{definition}
\begin{definition}
  The outer product $\ab \circ \bb$ of vectors $\ab \in \RR^{m}$ and
  $\bb \in \RR^{n}$ is given by a matrix $\Mb \in \RR^{m \times n}$ such
  that
  \begin{align}
    \label{eq:outer-matrix}
    m_{i,j} = a_{i} b_{j}. 
  \end{align}
  The definition can be extended to tensors by defining the outer
  product $\ab \circ \bb \circ \cbb$ of three vectors $\ab \in \RR^{m}$,
  $\bb \in \RR^{n}$, and $\cbb \in \RR^{p}$ as a tensor $\Mscrb \in
  \RR^{m \times n \times p}$ with
  \begin{align}
    \label{eq:outer-tensor}
    m_{i,j,k} = a_{i} b_{j} c_{k}. 
  \end{align}
\end{definition}
\begin{definition}
  Given a matrix $\Ab \in \RR^{n \times m}$, the linear operator
  $vec(\Ab)$ yields a vector $\xb \in \RR^{nm}$, which is obtained by
  stacking the columns of $\Ab$:
  \begin{align}
    \label{eq:vecdef}
    vec(\Ab) = \xb = \mymatrix{c}{
      \ab_{:,1} \\
      \ab_{:,2} \\
      \vdots \\
      \ab_{:,n}}.
  \end{align}
  Observe that 
  \begin{align}
    \label{eq:vecidx}
    x_{i+(j-1)n} = a_{i,j}. 
  \end{align}
  On the other hand, given a vector $\xb \in \RR^{nm}$, the operator
  $unvec_{(n,m)}(\xb)$ yields a matrix $\Ab \in \RR^{n \times m}$:
  \begin{align}
    \label{eq:unvecdef}
      unvec_{(n,m)}(\xb) = \Ab = \mymatrix{cccc}{
      \ab_{:,1} & \ab_{:,2} & \ldots & \ab_{:,n}}.
  \end{align}
\end{definition}
The Kronecker product satisfies the following well known relationship
(see \eg, proposition 7.1.9 of \cite{Bernstein05}):
\begin{align}
  \label{eq:vecabc}
  vec(\Ab \Bb \Cb) = \rbr{\Cb^{\top} \otimes \Ab} vec(\Bb). 
\end{align}
The Khatri-Rao product satisfies (see \eg, chapter 2 of \cite{SmiBroGel04}):
\begin{align}
  \label{eq:kr-product}
  \rbr{\Ab \odot \Bb}^{\top} \rbr{\Ab \odot \Bb}
   = \Ab^{\top} \Ab \ast  \Bb^{\top} \Bb.
\end{align}
Plugging this into the definition of the Moore-Penrose pseudo-inverse
\cite{Bernstein05} immediately shows that
\begin{align}
  \label{eq:kr-pseudo}
  \rbr{\Ab \odot \Bb}^{\dagger} = \rbr{\Ab^{\top} \Ab \ast \Bb^{\top}
    \Bb}^{-1} \rbr{\Ab \odot \Bb}^{\top}.
\end{align}

\subsection{An Example of Flattening Tensors}
\label{app:FlatteningTensors}
Let $\Xscrb$ be a $3 \times 4 \times 3$ tensor with frontal slices
\begin{align*}
  \mymatrix{cccc}{ \color{blue} 1 & \color{blue} 1
    & \color{blue} 4 & \color{blue} 2 \\
    \color{blue}3 & \color{blue}4 & \color{blue}5 & \color{blue}3 \\
    \color{blue}5 & \color{blue}0 & \color{blue}5 & \color{blue}1 }
  \mymatrix{cccc}{ \color{red}4 & \color{red}5 &
    \color{red}5 & \color{red}1 \\
    \color{red}1 & \color{red}1 & \color{red}1 & \color{red}4  \\
    \color{red}1 & \color{red}1 & \color{red}0 & \color{red}3 }
  \mymatrix{cccc}{ \color{black}1 & \color{black}0 &
    \color{black}2 & \color{black}4 \\
    \color{black}4 & \color{black}1 & \color{black}5 & \color{black}1 \\
    \color{black}5 & \color{black}2 & \color{black}4 & \color{black}1},
  \text{ then }
\end{align*}
\begin{align*}
  \Xb^{1} & = \mymatrix{cccc|cccc|cccc}{ \color{blue} 1 & \color{blue} 1 &
    \color{blue} 4 & \color{blue} 2 & \color{red}4 & \color{red}5 &
    \color{red}5 & \color{red}1 & \color{black}1 & \color{black}0 &
    \color{black}2 & \color{black}4 \\
    \color{blue}3 & \color{blue}4 & \color{blue}5 & \color{blue}3 &
    \color{red}1 & \color{red}1 & \color{red}1 & \color{red}4 &
    \color{black}4 & \color{black}1 & \color{black}5 & \color{black}1 \\
    \color{blue}5 & \color{blue}0 & \color{blue}5 & \color{blue}1 &
    \color{red}1 & \color{red}1 & \color{red}0 & \color{red}3 &
    \color{black}5 & \color{black}2 & \color{black}4 & \color{black}1}\\
  \Xb^{2} & = \mymatrix{ccc|ccc|ccc}{ \color{blue} 1 & \color{red} 4 &
    \color{black} 1 & \color{blue} 3 & \color{red} 1 & \color{black} 4 &
    \color{blue} 5 & \color{red} 1 & \color{black} 5  \\
    \color{blue} 1 & \color{red} 5 & \color{black} 0 & \color{blue} 4 &
    \color{red} 1 & \color{black} 1 & \color{blue} 0 & \color{red} 1 &
    \color{black} 2  \\
    \color{blue} 4 & \color{red} 5 & \color{black} 2 & \color{blue} 5 &
    \color{red} 1 & \color{black} 5 & \color{blue} 5 & \color{red} 0 &
    \color{black} 4  \\
    \color{blue} 2 & \color{red} 1 & \color{black} 4 & \color{blue} 3 &
    \color{red} 4 & \color{black} 1 & \color{blue} 1 & \color{red} 3 &
    \color{black} 1} \\
  \Xb^{3} & = \mymatrix{ccc|ccc|ccc|ccc}{ \color{blue} 1 & \color{blue}
    3 & \color{blue} 5 & \color{blue} 1 & \color{blue} 4 & \color{blue}
    0 & \color{blue} 4 & \color{blue} 5 & \color{blue} 5 & \color{blue}
    2 & \color{blue} 3 & \color{blue} 1 \\
    \color{red} 4 & \color{red} 1 & \color{red} 1 & \color{red} 5 &
    \color{red} 1 & \color{red} 1 & \color{red} 5 & \color{red} 1 &
    \color{red} 0 & \color{red} 1 & \color{red} 4 & \color{red} 3  \\
    \color{black} 1 & \color{black} 4 & \color{black} 5 & \color{black}
    0 & \color{black} 1 & \color{black} 2 & \color{black} 2 &
    \color{black} 5 & \color{black} 4 & \color{black} 4 & \color{black}
    1 & \color{black} 1 }
\end{align*}

\subsection{Proof of Lemma~\ref{lem:flat-vec}}
\label{app:proof-lem1}
\begin{proof}
  Using \eqref{eq:x1} and \eqref{eq:x1-blockIndices}, we can write 
  \begin{align*}
    x^{3}_{k,i+(j-1)I} = x^{1,k}_{i,j} = x^{1}_{i,j+(k-1)J}. 
  \end{align*}
  The result for $(n,n') = (1, 3)$ follows directly from
  \eqref{eq:vecidx} by letting $k=m$. For other values of $n$ and $n'$,
  the arguments are analogous.
\end{proof}

\newpage
\section{Review of ALS}
\label{sec:ALS}

In this section, we will introduce the CANDECOMP/PARAFAC(CP) 
decomposition model, and the ALS algorithm. The CP decomposition 
is a multi-way tensor factorization model. Given a tensor $\Xscrb \in 
\RR^{I \times J \times K}$, the $R$-rank CP decomposition of 
$\Xscrb$ is given by three matrices $\Ab \in \RR^{I \times R}$, $\Bb \in 
\RR^{J \times R}$, and $\Cb \in \RR^{K \times R}$ such that
\begin{align}
  \label{eq:parafac}
  \Xscrb \approx \sum_{r=1}^{R} \lambda_{r} \cdot \ab_{:,r} \circ \bb_{:,r} \circ
  \cbb_{:,r}.
\end{align}
Note that the columns of $\Ab$, $\Bb$, and $\Cb$ are normalized to have
unit length. The CP decomposition is computed by solving
\begin{align}
  \label{eq:tensorapprx}
  \min_{\Xscrbh} \nbr{\Xscrb - \Xscrbh} \quad \text{with} \quad \Xscrbh = 
    \sum_{r=1}^{R} \lambda_{r} \cdot \ab_{:,r} \circ \bb_{:,r} \circ
    \cbb_{:,r}. 
\end{align}
The most popular method to solve the above problem is the Alternating
Least Squares (ALS) algorithm \cite{KolBad09}. The basic idea here is to
fix all the matrices except one, and solve a least squares problem. Fixing
$\Bb$ and $\Cb$ and rewriting \eqref{eq:tensorapprx}, this amounts to
setting
\begin{align}
  \label{eq:als-astep}
  \Abhat \leftarrow \argmin_{\Abhat} \nbr{\Xb^{1} - 
    \Abhat \rbr{\Cb \odot \Bb}^{\top}}
\end{align}
The optimal solution of \eqref{eq:als-astep} can be rewritten using
\eqref{eq:kr-pseudo} as
\begin{align} 
  \label{eq:als-astep-sol}
  \Abhat & = \Xb^{1} \rbr{\rbr{\Cb \odot \Bb}^{\top}}^{\dagger} \\
  \label{eq:als-astep-sol-rw}
  & = \Xb^{1} \rbr{\Cb \odot \Bb}\rbr{\Cb^{\top}\Cb \ast
    \Bb^{\top}\Bb}^{-1}.
\end{align}
We obtain $\Ab$ by normalizing the columns of $\Abhat$. The ALS
procedure repeats analogously to find $\Bbhat$ and $\Cbhat$ until a
stopping criterion is met. The general CP-ALS algorithm is
summarized in Algorithm~\ref{algo:parfac}. 

\begin{algorithm}
  \SetKw{KwInit}{Initialize:}
  \SetKw{KwInput}{Input:}

  \KwInput $\Xb^{1}$, $\Xb^{2}$, $\Xb^{3}$

  \KwInit $\Ab$, $\Bb$, $\Cb$ 
  
  \While{stopping criterion not met}{
    $\Mb_1 \leftarrow \Xb^{1} \rbr{\Cb \odot \Bb}$
    \label{algo:x1cb-als}

    $\Ab \leftarrow \Mb_1 \rbr{\Cb^{\top} \Cb \ast \Bb^{\top} \Bb}^{-1}$

    Normalize columns of $\Ab$

    $\Mb_2 \leftarrow \Xb^{2} \rbr{\Ab \odot \Cb}$
    
    $\Bb \leftarrow \Mb_2 \rbr{\Ab^{\top} \Ab \ast \Cb^{\top} \Cb}^{-1}$
    
    Normalize columns of $\Bb$
    
    $\Mb_3 \leftarrow \Xb^{3} \rbr{\Bb \odot \Ab}$
        
    $\Cb \leftarrow \Mb_3 \rbr{\Bb^{\top} \Bb \ast \Ab^{\top} \Ab}^{-1}$
  
    Normalize columns of $\Cb$
  }
  \caption{CP-ALS algorithm}
  \label{algo:parfac}
\end{algorithm}

In tensor factorization, occasionally the problem of overfitting  occurs. Thus, 
we  add regularization terms to the objective function. Accordingly, we obtain 
the following new objective function:
\begin{align}
  \label{eq:tensorapprx-reg}
  \min_{\Xscrbh} \nbr{\Xscrb - \Xscrbh} + \frac{1}{2} \; \lambda 
    \rbr{\nbr{\Ab}^2 + \nbr{\Bb}^2 + \nbr{\Cb}^2} \quad \text{with} \quad 
    \Xscrbh = \sum_{r=1}^{R} \lambda_{r} \cdot \ab_{:,r} \circ \bb_{:,r} \circ
    \cbb_{:,r}. 
\end{align}
Then, the optimal solution of \eqref{eq:tensorapprx-reg} becomes
\begin{align}
  \label{eq:tensorapprx-ref-sol}
  \Abhat = \Xb^{1} \rbr{\Cb \odot \Bb}\rbr{\Cb^{\top}\Cb \ast
    \Bb^{\top}\Bb + \lambda \Ib}^{-1}.
\end{align}

\newpage
\section{Review of GD}
\label{sec:GD}

In this section, we will introduce the GD algorithm using 
CANDECOMP/PARAFAC(CP) decomposition model introduced in 
Section~\ref{sec:ALS}. This algorithm uses the same objective function as
CP-ALS except for normalization. Thus, we solve
\begin{align}
  \label{eq:gd-obj-fn}
  \min_{\Xscrbh} \sum_{i,j,k} \frac{1}{2} \rbr{x_{i,j,k} - \hat{x}_{i,j,k}}^{2}  
  \quad \text{ s.t. } \Xscrbh = \sum_{r=1}^{R} \; \ab_{:,r} \circ \bb_{:,r} 
    \circ \cbb_{:,r}
\end{align}
We can rewrite the equation in \eqref{eq:gd-obj-fn} as
\begin{align}
  \label{eq:gd-prob-A}
  f = \frac{1}{2} \nbr{\Xb^{1} - \Ab \rbr{\Cb \odot \Bb}^{\top}}^2.
\end{align}
Next, the gradient of \eqref{eq:gd-prob-A} with respect to $\Ab$
can be presented as
\begin{align}
  \label{eq:gd-A}
  \pwrt{\Ab}{f} = - \Xb^{1} \rbr{\Cb \odot \Bb} + 
    \Ab \rbr{\Cb^{\top} \Cb \ast \Bb^{\top} \Bb}.
\end{align}
In GD, the gradient of $f$ will be written as
\begin{align}
  \label{eq:grad-f}
  \nabla f= \mymatrix{c}{vec \rbr{\pwrt{\Ab}{f}} \\ 
    vec \rbr{\pwrt{\Bb}{f}} \\ vec \rbr{\pwrt{\Cb}{f}}}.
\end{align}
Then, we can compute the factor matrices $\Ab$, $\Bb$ and $\Cb$
with $\hat{f} = f - \alpha \; \nabla f$. The general CP-GD algorithm is
summarized in Algorithm~\ref{algo:gd}. 

\begin{algorithm}
  \SetKw{KwInit}{Initialize:}
  \SetKw{KwInput}{Input:}
  
  \KwInput $\Xb^{1}$, $\Xb^{2}$, $\Xb^{3}$

  \KwInit $\Ab$, $\Bb$, $\Cb$ 
  
  \While{stopping criterion not met}{
    $\Mb_1 \leftarrow \Xb^{1} \rbr{\Cb \odot \Bb}$
    \label{algo:x1cb-gd}
    
    $\nabla \Ab \leftarrow -\Mb_1+\Ab \rbr{\Cb^{\top} \Cb \ast \Bb^{\top} \Bb}$

    $\Mb_2 \leftarrow \Xb^{2} \rbr{\Ab \odot \Cb}$
    
    $\nabla \Bb \leftarrow -\Mb_2+\Bb \rbr{\Ab^{\top} \Ab \ast \Cb^{\top} \Cb}$
    
    $\Mb_3 \leftarrow \Xb^{3} \rbr{\Bb \odot \Ab}$
      
    $\nabla \Cb \leftarrow -\Mb_3+\Cb \rbr{\Bb^{\top} \Bb \ast \Ab^{\top} \Ab}$
      
    Calculate Step Size $\alpha$
    
    $\Ab \leftarrow \Ab - \alpha \nabla \Ab$
    
    $\Bb \leftarrow \Bb - \alpha \nabla \Bb$
    
    $\Cb \leftarrow \Cb - \alpha \nabla \Cb$
  }  
  \caption{CP-OPT algorithm}
  \label{algo:gd}
\end{algorithm}
We add regularization terms to the objective function to solve the problem of 
overfitting. The new objective function is now
\begin{align}
  \label{eq:gd-obj-fn-reg}
  \min_{\Xscrbh} \sum_{i,j,k} \frac{1}{2} \rbr{x_{i,j,k} - \hat{x}_{i,j,k}}^{2} 
  \frac{1}{2} \; \lambda \rbr{\nbr{\Ab}^2 + \nbr{\Bb}^2 + \nbr{\Cb}^2} 
  \quad \text{ s.t. } \Xscrbh = \sum_{r=1}^{R} \; \ab_{:,r} \circ \bb_{:,r} 
    \circ \cbb_{:,r}
\end{align}
Then, the gradient of \eqref{eq:gd-obj-fn-reg} with respect to $\Ab$ becomes
\begin{align}
  \label{eq:gd-A}
  \pwrt{\Ab}{f} = - \Xb^{1} \rbr{\Cb \odot \Bb} + 
    \Ab \rbr{\Cb^{\top} \Cb \ast \Bb^{\top} \Bb + \lambda \Ib}.
\end{align}

\newpage
\section{Illustrative Example}
\label{sec:IllustrDiffAlgor}

We illustrate the differences between our algorithm for computing $\Mb
:= \Xb^{1} \rbr{\Cb \odot \Bb}$ vs the algorithms proposed by
\cite{BadKol07} and \cite{KanPapHarFal12} on the following example:
Consider $\Xscrb \in \RR^{2 \times 3 \times 3}$ and let
\begin{align*}
  \Xb^{1}=\mymatrix{ccc|ccc|ccc}{1 & 0 & 6 & 0 & 4 & 7 & 2 & 0 & 0 \\
    0 & 0 & 0 & 3 & 0 & 8 & 0 & 5 & 9} \quad \text{ and } \quad
  \Xb^{2}=\mymatrix{ccc|ccc}{1 & 0 & 2 & 0 & 3 & 0 \\
    0 & 4 & 0 & 0 & 0 & 5 \\
    6 & 7 & 0 & 0 & 8 & 9}. 
\end{align*}
Moreover, let
\begin{align*}
  \Bb = \mymatrix{cc}{3 & 1 \\ 1 & 1 \\ 2 & 3} \text{ and } \Cb =
  \mymatrix{cc}{1 & 2 \\ 2 & 1 \\ 1 & 3}. 
\end{align*}
\cite{BadKol07} propose to store the above tensor as 
\begin{align*}
  \vb^{\Xscrb}&=\mymatrix{c}{1 \\ 2 \\ 3 \\ 4 \\ 5 \\ 6 \\ 7 \\ 8 \\
    9} \text{ and } \Sbb^{\Xscrb} = \mymatrix{ccc}{ 0 & 0 & 0 \\ 0 & 0 & 2
    \\ 1 & 0 & 1 \\ 0 & 1 & 1 \\ 1 & 1 & 2 \\ 0 & 2 & 0 \\ 0 & 2 & 1 \\
    1 & 2 & 1 \\ 1 & 2 & 2 },
\end{align*}
where $\vb^{\Xscrb}$ denotes the vector of non-zero entries of $\Xscrb$,
while $\Sbb^{\Xscrb}$ denotes the corresponding vector of indices. The
algorithm proposed in Sections 3.2.4 and 3.2.7 of \cite{BadKol07}
first computes
\begin{align*}
  \mb_{1} &= \mymatrix{c}{1 \\ 2 \\ 3 \\ 4 \\ 5 \\ 6 \\ 7 \\ 8 \\ 9} *
  \mymatrix{c}{3 \\ 3 \\ 3 \\ 1 \\ 1 \\ 2 \\ 2 \\ 2 \\ 2} *
  \mymatrix{c}{1 \\ 1 \\ 2 \\ 2 \\ 1 \\ 1 \\ 2 \\ 2 \\ 1}
  = \mymatrix{c}{3 \\ 6 \\ 18 \\ 8 \\ 5 \\ 12 \\ 28 \\ 32 \\ 18}.
\end{align*}
The above Hadamard product involves three vectors namely $\vb_{\Xscrb}$,
a vector formed by repeating entries of $\Bb_{:,1}$ based on
$\Sbb^{\Xscrb}_{:,2}$, and a vector formed by repeating entries of
$\Cb_{:, 1}$ based on $\Sbb^{\Xscrb}_{:,3}$. Similarly, we compute the
vector below but by using $\vb^{\Xscrb}$ and repeated entries from
$\Bb_{:,2}$ and $\Cb_{:, 2}$ respectively:
\begin{align*}    
  \mb_{2} &= \mymatrix{c}{1 \\ 2 \\ 3 \\ 4 \\ 5 \\ 6 \\ 7 \\ 8 \\ 9} *
  \mymatrix{c}{1 \\ 1 \\ 1 \\ 1 \\ 1 \\ 3 \\ 3 \\ 3 \\ 3} *
  \mymatrix{c}{2 \\ 3 \\ 1 \\ 1 \\ 3 \\ 2 \\ 1 \\ 1 \\ 3}
  = \mymatrix{c}{2 \\ 6 \\ 3 \\ 4 \\ 15 \\ 36 \\ 21 \\ 24 \\ 81}. 
\end{align*}
Finally, we use 
\begin{align*}
  \Sbb^{\Xscrb}_{:,1} = \mymatrix{c}{0 \\ 0 \\ 1 \\ 0 \\ 1 \\
    0 \\ 0 \\ 1 \\ 1} 
\end{align*}
to sum the appropriate entries of $\mb_{1}$ and $\mb_{2}$ to form
$\Mb$:
\begin{align*}
  \Mb &= \mymatrix{cc}{3+6+8+12+28 & 2+6+4+36+21 \\
    18+5+32+18 & 3+15+24+81} 
    = \mymatrix{cc}{57 & 69 \\ 73 & 123}. 
\end{align*}
The algorithm uses $2 \abr{\Omega^{\Xscrb}}$ extra storage and $5
\abr{\Omega^{\Xscrb}}$ flops to compute one column of $\Mb$.

On the other hand, the algorithm of \cite{KanPapHarFal12} computes
$\Mb$ as follows:
\begin{align*}
  \Nb_{1} &= \Xb^{1} * \rbr{\one_{I} \odot \rbr{\cbb_{:,0} \otimes
      \one_{J}}^{\top}} \\
  &=\mymatrix{ccccccccc}{1 & 0 & 6 & 0 & 4 & 7 & 2 & 0 & 0 \\
    0 & 0 & 0 & 3 & 0 & 8 & 0 & 5 & 9} * 
  \mymatrix{ccccccccc}{1 & 1 & 1 & 2 & 2 & 2 & 1 & 1 & 1 \\
    1 & 1 & 1 & 2 & 2 & 2 & 1 & 1 & 1} \\
  &= \mymatrix{ccccccccc}{1 & 0 & 6 & 0 & 8 & 14 & 2 & 0 & 0 \\
    0 & 0 & 0 & 6 & 0 & 16 & 0 & 5 & 9}.
\end{align*}
Here $\one_{n}$ denotes a vector of size $n$ with all entries set to
one. Similarly, if $bin \rbr{\Xb^{1}}$ denotes an indicator matrix for
the non-zero entries of $\Xb^{1}$, then
\begin{align*}
  \Nb_{2} &= bin \rbr{\Xb^{1}} * \rbr{\one_{I} \odot \rbr{\one_{K}
      \otimes \bb_{:,0}}^{\top}} \\
  &=\mymatrix{ccccccccc}{1 & 0 & 1 & 0 & 1 & 1 & 1 & 0 & 0 \\
    0 & 0 & 0 & 1 & 0 & 1 & 0 & 1 & 1} * 
  \mymatrix{ccccccccc}{3 & 1 & 2 & 3 & 1 & 2 & 3 & 1 & 2 \\
    3 & 1 & 2 & 3 & 1 & 2 & 3 & 1 & 2} \\
  &=\mymatrix{ccccccccc}{3 & 0 & 2 & 0 & 1 & 2 & 3 & 0 & 0 \\
    0 & 0 & 0 & 3 & 0 & 2 & 0 & 1 & 2}. 
\end{align*}
Finally we compute $\Nb_{3} = \Nb_{1} * \Nb_{2}$ via
\begin{align*}
  \Nb_{3}  =\mymatrix{ccccccccc}{3 & 0 & 12 & 0 & 8 & 28 & 6 & 0 & 0 \\
    0 & 0 & 0 & 18 & 0 & 32 & 0 & 5 & 18}
\end{align*}
to obtain
\begin{align*}
  \mb_{:,1} & = \Nb_{3} \; \one_{JK} = \mymatrix{c}{57 \\ 73}. 
\end{align*}
To compute the second column of $\Mb$ we use 
\begin{align*}
  \Nb_{1} &= \Xb^{1} * \rbr{\one_{I} \odot \rbr{\cbb_{:,1}
      \otimes \one_{J}}^{\top}} \\
  &=\mymatrix{ccccccccc}{1 & 0 & 6 & 0 & 4 & 7 & 2 & 0 & 0 \\
    0 & 0 & 0 & 3 & 0 & 8 & 0 & 5 & 9} * 
  \mymatrix{ccccccccc}{2 & 2 & 2 & 1 & 1 & 1 & 3 & 3 & 3 \\
    2 & 2 & 2 & 1 & 1 & 1 & 3 & 3 & 3} \\
  &=\mymatrix{ccccccccc}{2 & 0 & 12 & 0 & 4 & 7 & 6 & 0 & 0 \\
    0 & 0 & 0 & 3 & 0 & 8 & 0 & 15 & 27}.
\end{align*}
\begin{align*}
  \Nb_{2} &= bin \rbr{\Xb^{1}} * \rbr{\one_{I} \odot \rbr{\one_{K}
      \otimes \bb_{:,1}}^{\top}} \\
  &=\mymatrix{ccccccccc}{1 & 0 & 1 & 0 & 1 & 1 & 1 & 0 & 0 \\
    0 & 0 & 0 & 1 & 0 & 1 & 0 & 1 & 1} *
  \mymatrix{ccccccccc}{1 & 1 & 3 & 1 & 1 & 3 & 1 & 1 & 3 \\
    1 & 1 & 3 & 1 & 1 & 3 & 1 & 1 & 3} \\
  &=\mymatrix{ccccccccc}{1 & 0 & 3 & 0 & 1 & 3 & 1 & 0 & 0 \\
    0 & 0 & 0 & 1 & 0 & 3 & 0 & 1 & 3}. 
\end{align*}
Finally we compute $\Nb_{3} = \Nb_{1} * \Nb_{2}$ via 
\begin{align*}
  \Nb_{3} &=\mymatrix{ccccccccc}{2 & 0 & 36 & 0 & 4 & 21 & 6 & 0 & 0 \\
    0 & 0 & 0 & 3 & 0 & 24 & 0 & 15 & 81} 
\end{align*}
and then compute 
\begin{align*}  
  \mb_{:,1} = \Nb_{3} \; \one_{JK} = \mymatrix{c}{69 \\ 123}.
\end{align*}
The algorithm uses $\max \rbr{J+\abr {\Omega^{\Xscrb}},
  K+\abr{\Omega^{\Xscrb}}}$ extra storage and $5 \abr{\Omega^{\Xscrb}}$
flops to compute one column of $\Mb$.

In contrast, our algorithm computes $\Mb$ as follows:
\begin{align*}
  \mb_{:,0} &= unvec_{\rbr{2,3}} \rbr{\mymatrix{ccc}{1 & 0 & 6 \\ 0 & 4 & 7 \\
      2 & 0 & 0 \\ \hline 0 & 0 & 0 \\ 3 & 0 & 8 \\ 0 & 5 & 9}
    \mymatrix{c}{3 \\ 1 \\ 2}}^{\top} \mymatrix{c}{1 \\ 2 \\ 1}
  = \mymatrix{ccc}{15 & 18 & 6 \\ 0 & 25 & 23} \mymatrix{c}{1 \\ 2 \\ 1} 
  = \mymatrix{c}{57 \\ 73} \\
  \mb_{:,1} &= unvec_{\rbr{2,3}} \rbr{\mymatrix{ccc}{1 & 0 & 6 \\ 0 & 4 & 7 \\
      2 & 0 & 0 \\ \hline 0 & 0 & 0 \\ 3 & 0 & 8 \\ 0 & 5 & 9}
    \mymatrix{c}{1 \\ 1 \\ 3}}^{\top} \mymatrix{c}{2 \\ 1 \\ 3}
  =\mymatrix{ccc}{19 & 25 & 2 \\ 0 & 27 & 32} \mymatrix{c}{2 \\ 1 \\ 3}
  = \mymatrix{c}{69 \\ 123}.
\end{align*}
Our algorithm only requires $nnzc(\Xb^{2})$ extra storage space and
$2\abr{\Omega^{\Xscrb}}$ flops for computing $\Mb$.

\newpage
\section{The ALS and GD algorithms of the DFacTo model}
\label{sec:dfacto-alsgd}

The ALS and GD algorithms of the DFacTo model 
(Section~\ref{sec:NewModel}) is summarized in 
Algorithms~\ref{algo:newals} and ~\ref{algo:newgd}.
We can solve the problem of overfitting by adding a $\lambda \Ib$ term in 
$\Cb^{\top} \Cb * \Bb^{\top} \Bb$, $\Ab^{\top} \Ab * \Cb^{\top} \Cb$, and 
$\Bb^{\top} \Bb * \Ab^{\top} \Ab$ of Algorithms~\ref{algo:newals} 
(lines~\ref{alg:A-als}, ~\ref{alg:B-als}, ~\ref{alg:C-als}) and ~\ref{algo:newgd} (lines~\ref{alg:A-gd}, ~\ref{alg:B-gd}, ~\ref{alg:C-gd}).

\begin{algorithm}
  \SetKw{KwInit}{Initialize:}
  \SetKw{KwInput}{Input:}
  
  \KwInput $\Xb^{1}$, $\Xb^{2}$, $\Xb^{3}$
  
  \KwInit $\Ab$, $\Bb$, $\Cb$ 
  
  \While{stopping criterion not met}{
    \While{r=1, 2,\ldots, R}{
      $\nbb_{:,r} \leftarrow unvec_{(K,I)} \rbr{\rbr{\Xb^{2}}^{\top}
        \bb_{:,r}}^{\top} \; \cbb_{:,r}$
    }

    $\Ab \leftarrow \Nb \rbr{\Cb^{\top} \Cb * \Bb^{\top} \Bb}^{-1}$
    \label{alg:A-als}

    Normalize columns of $\Ab$

    \While{r=1, 2,\ldots, R}{
      $\nbb_{:,r} \leftarrow unvec_{(I,J)} \rbr{\rbr{\Xb^{3}}^{\top}
        \cbb_{:,r}}^{\top} \; \ab_{:,r}$
    }

    $\Bb \leftarrow \Nb \rbr{\Ab^{\top} \Ab * \Cb^{\top} \Cb}^{-1}$
    \label{alg:B-als}

    Normalize columns of $\Bb$

    \While{r=1, 2,\ldots, Right}{
      $\nbb_{:,r} \leftarrow unvec_{(J,K)} \rbr{\rbr{\Xb^{1}}^{\top}
        \ab_{:,r}}^{\top} \; \bb_{:,r}$
    }

    $\Cb \leftarrow \Nb \rbr{\Bb^{\top} \Bb * \Ab^{\top} \Ab}^{-1}$
    \label{alg:C-als}

    Normalize columns of $\Cb$
  }
  \caption{DFacTo(ALS) algorithm for Tensor Factorization}
  \label{algo:newals}
\end{algorithm}

\begin{algorithm}
  \SetKw{KwInit}{Initialize:}
  \SetKw{KwInput}{Input:}
  
  \KwInput $\Xb^{1}$, $\Xb^{2}$, $\Xb^{3}$
  
  \KwInit $\Ab$, $\Bb$, $\Cb$ 
  
  \While{stopping criterion not met}{
    \While{r=1, 2,\ldots, R}{
      $\nbb_{:,r} \leftarrow unvec_{(K,I)} ((\Xb^2)^{\top}
        \bb_{:,r})^{\top} \; \cbb_{:,r}$
    }

    $\nabla \Ab \leftarrow \Nb + 
      \Ab \rbr{\Cb^{\top} \Cb * \Bb^{\top} \Bb}$
    \label{alg:A-gd}

    \While{r=1, 2,\ldots, R}{
      $\nbb_{:,r} \leftarrow unvec_{(I,J)} \rbr{\rbr{\Xb^{3}}^{\top}
        \cbb_{:,r}}^{\top} \; \ab_{:,r}$
    }

    $\nabla \Bb \leftarrow \Nb + 
      \Bb \rbr{\Ab^{\top} \Ab * \Cb^{\top} \Cb}$
      \label{alg:B-gd}

    \While{r=1, 2,\ldots, Right}{
      $\nbb_{:,r} \leftarrow unvec_{(J,K)} \rbr{\rbr{\Xb^{1}}^{\top}
        \ab_{:,r}}^{\top} \; \bb_{:,r}$
      \label{alg:C-gd}
    }

    $\nabla \Cb \leftarrow \Nb + 
      \Cb \rbr{\Bb^{\top} \Bb * \Ab^{\top} \Ab}$
      
    $\alpha \leftarrow Linesearch(\Ab, \Bb, \Cb, \nabla \Ab,
      \nabla \Bb, \nabla \Cb$)
    
    $\Ab \leftarrow \Ab - \alpha \nabla \Ab$ \\
    $\Bb \leftarrow \Bb - \alpha \nabla \Bb$ \\
    $\Cb \leftarrow \Cb - \alpha \nabla \Cb$ \\
  }
  \caption{DFacTo(GD) algorithm for Tensor Factorization}
  \label{algo:newgd}
\end{algorithm}

\section{Joint Matrix Completion and Tensor Factorization}
\label{sec:JointMatcomTenfac}
Generally, matrix completion is used when predicting how users will
rate items based on data of how these users have previously rated
other items. Occasionally, however, the accuracy of prediction from
matrix completion is poor because matrix completion only uses prior
information on the user, item, and rating. Thus, we suggest a joint
matrix completion and tensor factorization model. In this model, we
add a word count tensor $\Xscrb$ with user-item-word dimensions to
the previous rating matrix $\Yb$. This model is similar to
\cite{AcaKolDun11}; but instead of sharing just one dimension (item),
we introduce a model that shares both the user and item dimensions. 
Also, while \cite{AcaKolDun11} applies joint tensor completion and
matrix factorization, we suggest using joint matrix completion and
tensor factorization.

Our joint model can be computed by solving
\begin{align}
  \label{eq:joint-obj-fn}
  \min_{\Xscrbh, \Ybhat} & \sum_{(i,j) \in \Omega^{\Yb}}
    \frac{1}{2} \rbr{y_{i,j} - \hat{y}_{i,j}}^{2} + \mu \sum_{i,j,k}
    \frac{1}{2} \rbr{x_{i,j,k} - \hat{x}_{i,j,k}}^{2}  + \lambda \frac{1}{2}
    \rbr{\nbr{\Ab}^2 + \nbr{\Bb}^2 + \nbr{\Cb}^2} \\
  \text{ s.t. } & \Xscrbh = \sum_{r=1}^{R} \; \ab_{:,r} \circ \bb_{:,r} 
    \circ \cbb_{:,r}, \Ybhat = \sum_{r=1}^{R} \; \ab_{:,r} \circ \bb_{:,r}
    \nonumber
\end{align}
We can rewrite the equation in \eqref{eq:joint-obj-fn} as
\begin{align}
  \label{eq:joint-prob-A}
  f = \frac{1}{2} \sum_{j \in \Omega_{i,:}^{\Yb}} \rbr{y_{i,j} - \ab_{i,:}
    \bb_{j,:}^{\top}}^{2} + \frac{1}{2} \mu \sum_{j} \rbr{x^{1}_{i,j} - \ab_{i,:}
    \rbr{\Cb \odot \Bb}_{j,:}^{\top}}^2 + \frac{1}{2} \lambda \ab_{i,:}
    \ab_{i,:}^{\top}.
\end{align}
Next, the gradient of \eqref{eq:joint-prob-A} with respect to $\ab_{i,:}$
can be presented as
\begin{align}
  \label{eq:grad-joint-A}
  \pwrt{\ab_{i,:}}{f} = - \sbr{\yb_{i,:} \Bb + \mu \; \xb^{1}_{i,:} 
    \rbr{\Cb \odot \Bb}} + \ab_{i,:} \sbr{\sum_{j \in \Omega^{\Yb}}
    \bb_{j,:}^{\top} \bb_{j,:} + \mu \; \Cb^{\top} \Cb \ast \Bb^{\top} \Bb +
    \lambda \Ib}.
\end{align}
The two optimization methods we use to solve the minimization problem
in this paper are the Gradient Descent (GD) and the Alternative Least
Squares (ALS).  

In GD, the gradient of $f$ will be written as
\begin{align}
  \label{eq:grad-f}
  \nabla f= \mymatrix{c}{vec \rbr{\pwrt{\Ab}{f}} \\ 
    vec \rbr{\pwrt{\Bb}{f}} \\ vec \rbr{\pwrt{\Cb}{f}}}.
\end{align}
And each $vec(\cdot)$ of \eqref{eq:grad-f} will be computed by the
gradient of $f$ in \eqref{eq:grad-joint-A} that corresponds to 
$\ab_{j,:}$, $\bb_{j,:}$ and $\cbb_{k,:}$, respectively because
\begin{align*}
  vec \rbr{\pwrt{\Ab}{f}}=\mymatrix{c}{\rbr{\pwrt{\ab_{1,:}}{f}}^{\top} \\ 
    \rbr{\pwrt{\ab_{2,:}}{f}}^{\top} \\ \vdots \\ \rbr{\pwrt{\ab_{I,:}}{f}}^{\top}}.
\end{align*}
Then, we can compute the factor matrices $\Ab$, $\Bb$ and $\Cb$
with $\hat{f} = f - \alpha \; \nabla f$.

On the other hand, in ALS, setting \eqref{eq:grad-joint-A} to zero shows
that the optimal solution of \eqref{eq:joint-prob-A} is given by
\begin{align*}
  \abhat_{i,:} = \sbr{\yb_{i,:} \Bb + \mu \; \xb^{1}_{i,:} \rbr{\Cb \odot \Bb}}
    \sbr{\sum_{j \in \Omega^{\Yb}} \bb_{j,:}^{\top} \bb_{j,:} + \mu \;
    \Cb^{\top} \Cb \ast \Bb^{\top} \Bb + \lambda \Ib}^{-1}. 
\end{align*}

In both cases, we will use DFacTo, which we suggested in Section~\ref{sec:NewModel}, to avoid the intermediate data explosion problem
of $\Xb^{1} (\Cb \odot \Bb)$.

\subsection{Experimental Evaluation}
\label{sec:ExperEval}
We evaluate the joint tensor factorization and matrix completion model
on a subset of datasets from Table~\ref{tab:dataset}. Arguably, our
experimental evaluation is very preliminary, but promising. The
experimental setup is as follows: We split each dataset into train,
test, and validation. We randomly select 60\% of review, rating pairs
and designate them as training data. We then select 20\% of the
remaining review, rating pairs, discard the reviews, remove users or
items which do not occur in the training data, and use it for
validation. A similar procedure is used to generate the test dataset.
Cellartracker and RateBeer datasets contain ratings which are not in a 0
to 5 scale. For consistency, we normalize these ratings to be in 0 to 5.
Our evaluation metric is the mean square error which is given by
$\sum{(i,j) \in \Omega^{\Yb} (y_{i,j} - \hat{y}_{i,j})}$, were $y_{i,j}$
is a test rating and $\hat{y}_{i,j}$ is the rating predicted by our
model.

We train our model with $\mu \in \cbr{10^{2}, 10^{1}, ... , 10^{-9},
10^{-10}}$ and $\lambda \in \cbr{100, 10, 1, 0.1, 0.01}$, evaluate its
performance on the validation set, and pick the best model based on its
mean square error. We use this model to predict on the test dataset and
report average mean square error. In Tables~\ref{tab:mse-gd} and 
\ref{tab:mse-als}, we show the MSEs from both the matrix completion
and our joint model using GD and ALS. For GD, the method of
backtracking line search was used.

\begin{table}[h]
  \centering
  \begin{tabular}{c|r|r|r|r|r}
    \hline
    \multirow{2}{*}{Dataset} & \multicolumn{2}{c|}{Matrix Completion} & 
      \multicolumn{3}{c}{Joint (MC + TF)} \\
    \hhline{~-----}
    & \multicolumn{1}{c|}{$\lambda$} & Test MSE &
      \multicolumn{1}{c|}{$\mu$} & \multicolumn{1}{c|}{$\lambda$} &
        Test MSE \\
    \hline
    Yelp Phoenix & 10 & 3.133650 & $10^{-6}$ & 0.1 & 1.481320 \\
    Cellartracker & 1 & 1.506590 & $10^{-7}$ & 1 & 0.927066 \\
    Beeradvocate  & 1 & 0.603431 & $10^{-7}$ & 0.1 & 0.459174 \\
    Ratebeer & 0.01 & 0.390188 & $10^{-9}$ & 1 & 0.389653 \\
    \hline
  \end{tabular}
  \caption{Best Test MSE of single matrix completion and joint matrix
    completion and tensor factorization model after 500 iterations using
    Gradient Descent.}
  \label{tab:mse-gd}
\end{table}
\begin{table}[h]
  \centering
  \begin{tabular}{c|r|r|r|r|r}
    \hline
    \multirow{2}{*}{Dataset} & \multicolumn{2}{c|}{Matrix Completion} & 
      \multicolumn{3}{c}{Joint (MC + TF)} \\
    \hhline{~-----}
    & \multicolumn{1}{c|}{$\lambda$} & Test MSE &
      \multicolumn{1}{c|}{$\mu$} & \multicolumn{1}{c|}{$\lambda$} &
        Test MSE \\
    \hline
    Yelp Phoenix & 1 & 2.904320 & 1 & 1 & 1.944050 \\
    Cellartracker & 1 & 1.148010 & 100 & 0.01 & 0.363496 \\
    Beeradvocate & 0.1 & 0.465695 & 10 & 0.1 & 0.373827 \\
    Ratebeer & 0.1 & 0.355989 & 0.1 & 1 & 0.318692 \\
    \hline
  \end{tabular}
  \caption{Best Test MSE of single matrix completion and joint matrix
    completion and tensor factorization model after 500 iterations using
    ALS.}
  \label{tab:mse-als}
\end{table}

The results show that our joint model produces better MSEs than matrix
completion across all datasets and methods. All in all, our joint model
improves the accuracy of prediction when compared to matrix
completion.

\end{document}